\documentclass{article}

\PassOptionsToPackage{numbers, compress}{natbib}

\usepackage[preprint]{neurips_2026}

\usepackage[utf8]{inputenc} 
\usepackage[T1]{fontenc}    
\usepackage{hyperref}       
\usepackage{url}            
\usepackage{booktabs}       
\usepackage{amsfonts}       
\usepackage{nicefrac}       
\usepackage{microtype}      
\usepackage{xcolor}         
\usepackage{amsmath,amssymb,amsthm, natbib}
\usepackage{graphicx}
\usepackage{multirow}
\usepackage[table]{xcolor}
\usepackage{wrapfig}
\usepackage{booktabs}
\theoremstyle{plain}
\newtheorem{theorem}{Theorem}[section]
\newtheorem{proposition}[theorem]{Proposition}

\theoremstyle{definition}

\theoremstyle{remark}

\title{Why Semantic Entropy Fails: Geometry-Aware and Calibrated Uncertainty for Policy Optimization}

\author{%
    Zheyuan Zhang\textsuperscript{1}, 
    Kaiwen Shi\textsuperscript{1}, 
    Han Bao\textsuperscript{1},
    Zehong Wang\textsuperscript{1}, 
    Tianyi Ma\textsuperscript{1},
    Yanfang Ye\textsuperscript{1}\textsuperscript{$\dagger$}\\
    \textsuperscript{1}University of Notre Dame, 
    \textsuperscript{$\dagger$}Corresponding Author
    \\
    \texttt{\{zzhang42, yye7\}@nd.edu}
}

\begin{document}

\maketitle

\begin{abstract}
Post-training has become central to improving reasoning and alignment in large language models, where critic-free models enable scalable learning from model-generated outputs but lack principled mechanisms to distinguish informative from noisy signals. Recent approaches leverage response-level measures as uncertainty signals to regulate group-based optimization methods such as GRPO. Yet their empirical success remains unstable and unclear in how they influence optimization dynamics. In this paper, we provide, to our knowledge, the first principled formulation that interprets uncertainty signals as mechanisms for characterizing and regulating gradient variance and learning signal quality. Based on both empirical and theoretical analysis, we identify two critical gaps of current entropy-based estimators: \textbf{\emph{The anisotropic gap}} and \textbf{\emph{The calibration gap}}. Motivated by this analysis, we propose \textbf{Geometric-aware Calibrated Policy Optimization (GCPO)}, a novel framework integrating geometry-aware measures to capture semantic disagreement with reward-based calibration to align uncertainty with learning signal strength. Experiments on multiple benchmarks show that our approach more faithfully tracks gradient variability and consistently improves post-training performance. Our results highlight the importance of designing uncertainty signals that are aligned with optimization dynamics, offering a principled perspective for robust post-training.
\end{abstract}

\section{Introduction}

Despite the remarkable success of large-scale pretraining, recent Large Language Model (LLM) development increasingly shifts toward post-training to elicit reasoning, alignment, and task-specific behaviors \cite{ye2025llms4all,tong2024dart, rafailov2023direct,chen2025clear, team2025kimi,wang2026reasoning,li2026graph}. In this regime, the central challenge is not only how to scale model capacity, but how to reliably extract learning signals from environments or model-generated outputs. A growing line of work, represented by Group Relative Policy Optimization (GRPO) \cite{shao2024deepseekmath}, adopts a critic-free, group-based formulation, where multiple responses are sampled for each input and optimized using relative rewards \cite{liu2025understanding, ma2026non, ahmadian2024back}. While this paradigm significantly simplifies supervision and improves scalability, it introduces a key limitation: all queries are treated uniformly, despite the fact that different queries can induce vastly different levels of uncertainty and learning signal quality \cite{dong2025agentic, dong2026agentic}.

To address this issue, a promising direction is to introduce uncertainty signals as a mechanism to select high quality model-generated outputs and regulate learning accordingly. Early approaches rely on token-level entropy to quantify uncertainty in the model distribution \cite{fadeeva2023lm, fadeeva2024fact,wang2024gft}. More recent work extends uncertainty estimation to the response level, forming a family of \emph{semantic entropy-type measures}, including semantic entropy \cite{chen2025seed, zhang2025right}, predictive entropy \cite{duan2025uprop}, kernel-based variants \cite{nikitin2024kernel}, and self-consistency scores \cite{chen2025grpo, wen2025scenario}. These methods capture uncertainty over meaning rather than surface form, showing empirical benefits in improving reasoning performance and training stability. \textbf{However, despite their empirical success, the role of such uncertainty signals in optimization remains poorly understood.} Existing approaches largely treat uncertainty as a heuristic proxy, without explicitly modeling how it interacts with gradient updates and learning dynamics. \textbf{As a result, it remains unclear whether these signals faithfully reflect gradient variance and learning signal quality.} If poorly handled, such methods may simply suppress high-variance informative samples, thereby collapsing gradient signal and potentially leading to suboptimal or failed policy learning.

In this paper, we provide, to the best of our knowledge, \textbf{the first principled analysis that explicitly formulates uncertainty signals as mechanisms for characterizing and regulating gradient variance and learning signal quality.} Based on both theoretical and empirical statistical analysis, we reveal two fundamental gaps of existing methods under this formulation: 1) \textbf{\emph{The anisotropic gap}}, where entropy-based measures capture probability dispersion but ignore the geometric magnitude of semantic disagreement, although different errors may induce substantially different gradient directions. 2) \textbf{\emph{The calibration gap}}, where entropy is decoupled from reward informativeness, and therefore cannot distinguish uninformative uncertainty from useful reward-differentiated variation. 

\textbf{Motivated by the analysis,} we then propose \textbf{G}eometric-aware reward-\textbf{C}alibrated \textbf{P}olicy \textbf{O}ptimization \textbf{(GCPO)}. Instead of treating uncertainty as purely distributional dispersion, our framework explicitly models uncertainty as a mechanism for \emph{characterizing and calibrating gradient variance with respect to learning signal quality}. Specifically, we introduce geometry-aware measures, including \emph{Cosine Dispersion (CD) and Barycentric Transport (BoT)}, to capture semantic disagreement beyond entropy, and further incorporate a \emph{Reward Dispersion (RD)} module to align update strength with reward informativeness. Extensive experiments on QA and math reasoning benchmarks show that our proposed measures better align with sample-level gradient variance and consistently improve GRPO-style optimization over entropy-based baselines. These findings provide empirical evidence for our formulation, suggesting a principled guideline for designing uncertainty signals that preserve informative gradients while suppressing spurious variance in LLM post-training. Overall, our contributions can be summarized as follows.

\begin{itemize}
    \item We provide, to our best knowledge, the first principled analysis that explicitly formulates entropy-type uncertainty signals as mechanisms for characterizing and regulating gradient variance and learning signal quality, revealing two fundamental limitations: \textbf{\emph{the anisotropic gap}} and \textbf{\emph{the calibration gap}}, which explain both their effectiveness and failure modes.

    \item We propose \textbf{G}eometric-aware reward-\textbf{C}alibrated \textbf{P}olicy \textbf{O}ptimization \textbf{(GCPO)}, a framework that models uncertainty as a mechanism for \emph{calibrating gradient variance}, combining geometry-aware semantic structure with reward-based calibration to selectively suppress noisy variance while preserving informative learning signals.

    \item Extensive experiments show that our method achieves stronger alignment and consistently improves GRPO-style training across multiple tasks, supporting a principled guideline for designing uncertainty signals preserving informative gradients while avoiding collapse.
\end{itemize}

\section{Why Semantic Entropy-based Measures Fail?}
\label{sec:preliminary}

\subsection{Preliminary}

\textbf{GRPO. } 
Group Relative Policy Optimization (GRPO) is a critic-free policy gradient method for reasoning-oriented RL \cite{shao2024deepseekmath}. For each input $x$, it samples a group of responses $Y=\{y_i\}_{i=1}^G$ and estimates the gradient using normalized rewards:
\begin{equation}
\hat{g}(x) = \frac{1}{G} \sum_{i=1}^{G} \hat{A}_i \, \nabla_\theta \log \pi_\theta(y_i \mid x), \quad
\hat{A}_i = \frac{R(y_i) - \mu_R}{\sigma_R + \epsilon}.
\end{equation}

\vspace{-5pt}

\textbf{Why Focus on GRPO?}
GRPO is known for its strong and widely adopted optimization paradigm and its critic-free, rollout-based formulation makes it particularly well-aligned with the study of both \emph{token-level and response-level uncertainty}. Unlike actor-critic methods that rely on learned value functions for variance reduction, GRPO exposes gradient variance at the query level and provides a natural setting for analyzing how response-level uncertainty influences optimization.

\textbf{Semantic Entropy-type Measures. }
Token-level entropy measures uncertainty over surface forms, and therefore is sensitive to paraphrastic variation that does not affect meaning. As a result, it often overestimates uncertainty in reasoning tasks where multiple responses differ only in wording but share the same underlying semantics \cite{chen2025seed, nikitin2024kernel}. To address this limitation, semantic entropy \cite{kuhn2023semantic} measures uncertainty at the response level by aggregating semantically equivalent outputs, thereby focusing on variability across distinct meanings rather than surface realizations. This formulation has led to a broader class of response-level entropy-like methods, including predictive entropy, kernel-based variants, and other consistency-based scores, all of which quantify uncertainty via grouped responses rather than token-level distributions. Formally, we denote this family by $H_{sem}$. For instance, in semantic entropy, $H_{\mathrm{sem}}(x) = - \sum_{k=1}^{K} C_k \log C_k$, where $C_k$ denotes the probability mass of semantic cluster $k$. In predictive entropy, $H_{\mathrm{sem}}(x) = - \sum_{k=1}^{K} q_k \log q_k$, where $q_k$ represents the mass of prediction groups. In this work, for simplicity, we treat semantic entropy as a representative instance of this class. \textbf{Note that despite our focus on response-level entropy, the failure patterns analyzed below also apply for token entropy. Therefore we may use entropy to refer to both measures for simplicity unless otherwise explained.}

\begin{figure*}[t]
	\centering
	\includegraphics[width=1\linewidth]{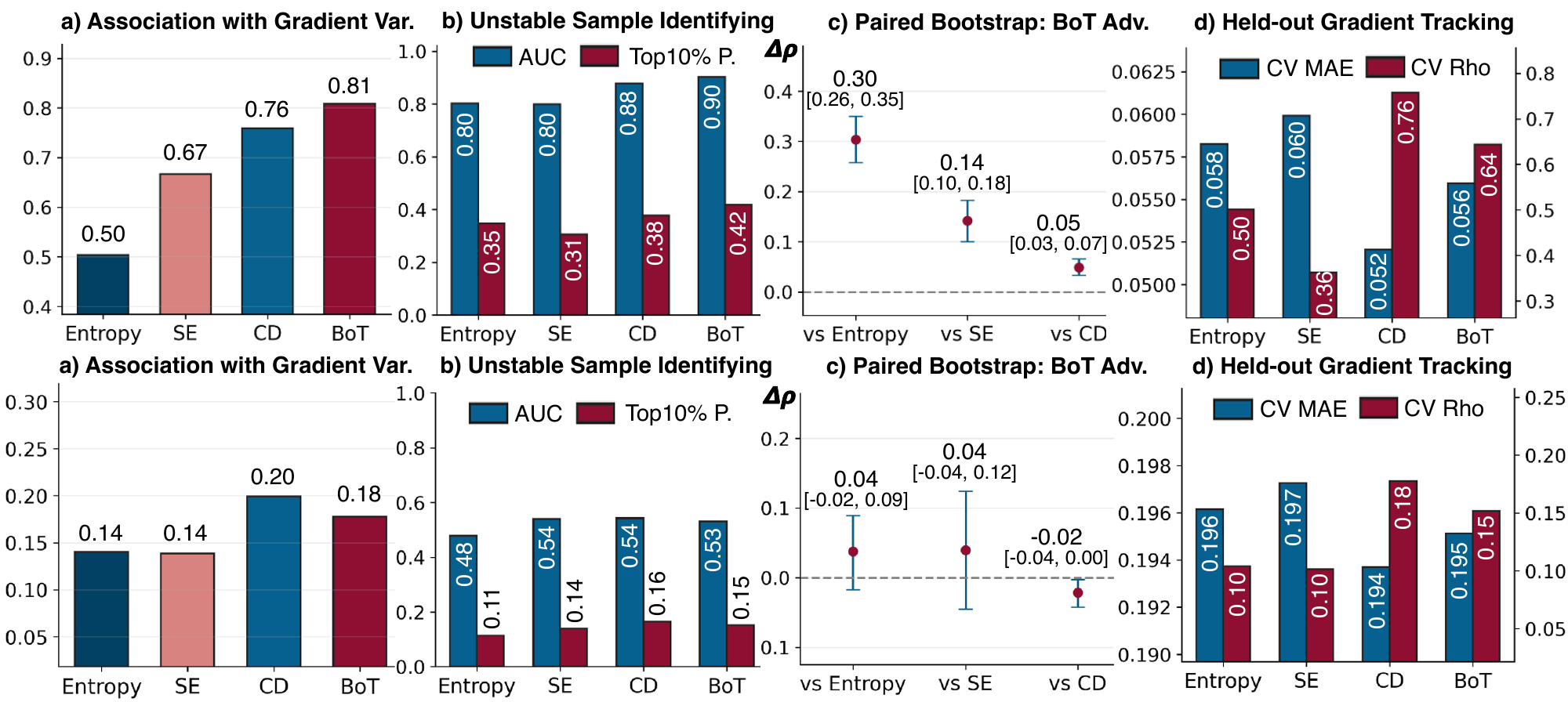}
        \vspace{-15pt}
	\caption{Statistical analysis of uncertainty vs. gradient variance. Top row: NarrativeQA; bottom row: Qasper. \textbf{Overall, geometry-aware measures (CD, BoT) better align with gradient variance than entropy-based measures, with stronger effects on NarrativeQA.} Due to space constraints, full experimental details of statistical analysis are provided in Appendix~\ref{app:stats}.
    }
        \vspace{-15pt}
    \label{fig:prelim}
\end{figure*}

\subsection{The Anisotropic Gap}
\label{sec:a_gap}

Prior research has established that the stability and quality of policy optimization are governed by the signal-to-noise ratio of gradient estimates, making variance control a key factor for stable and effective reinforcement learning \cite{sutton2018reinforcement, cui2025entropy, roberts2008signal}. To understand how response-level uncertainty interacts with optimization from this perspective, we analyze the structure of gradient variance under semantic grouping. Let $Y \sim \pi_\theta(\cdot \mid x)$ denote a sampled response, and let $Z = c(Y) \in \{1,\dots,K\}$ denote its semantic cluster label. Since GRPO operates on sampled responses, variability across distinct semantic modes directly influences reward differences, thus gradient updates, making response-level uncertainty particularly relevant. Let $g(Y) = \nabla_\theta \log \pi_\theta(Y \mid x)$. By the law of total covariance \cite{casella2024statistical},
\begin{equation}
\mathrm{Cov}(g \mid x)
=
\mathbb{E}\!\left[\mathrm{Cov}(g \mid Z, x)\mid x\right]
+
\mathrm{Cov}\!\left(\mathbb{E}[g \mid Z, x]\mid x\right).
\end{equation}
Taking the trace yields
\begin{equation}
\mathcal{V}(x)
=
\underbrace{\sum_{k=1}^{K} p_k(x)\,\mathrm{Tr}\!\big(\mathrm{Cov}(g \mid Z = k, x)\big)}_{\mathcal{V}_{\mathrm{intra}}(x)}
+
\underbrace{\mathrm{Tr}\!\big(\mathrm{Cov}(\mu_Z \mid x)\big)}_{\mathcal{V}_{\mathrm{inter}}(x)},
\end{equation}
where $p_k(x) := \mathbb{P}(Z = k \mid x)$ and $\mu_k := \mathbb{E}[g \mid Z = k, x]$. This decomposition separates gradient variability into intra-cluster variation within each semantic mode and inter-cluster variation across modes. Since $H_{\mathrm{sem}}$ depends only on the distribution $\{p_k(x)\}$, it primarily captures dispersion across semantic modes and is therefore more aligned with the inter-cluster component of gradient variance.

\begin{proposition}
\label{prop:a_gap}
\textbf{The Anisotropic Gap.} 
To be specific, SE-type measures depend only on the cluster mass distribution $\{\mathbb{P}(Z = k \mid x)\}_{k=1}^K$, and therefore can only reflect coarse inter-cluster dispersion, while discarding the intra-cluster variance term $\mathcal{V}_{\mathrm{intra}}(x)$, effectively assuming an \textit{isotropic error space} where all disagreements are equally severe.
\vspace{-5pt}
\end{proposition}

\textbf{In reality, reasoning errors are highly \textit{anisotropic}}: a near-miss (e.g., a minor mistake in an otherwise correct reasoning chain) induces a gradient that remains largely aligned with the correct solution, whereas a fundamentally incorrect reasoning path leads to a substantially different, often conflicting, gradient direction. Since $H_{\mathrm{sem}}$ ignores such intra-cluster variability, it cannot distinguish between these cases. As a result, controlling $H_{\mathrm{sem}}$ only constrains coarse mode dispersion while providing a loose surrogate for the true gradient variance $\mathcal{V}(x)$, leaving potentially significant optimization-relevant variation unaccounted for. \textbf{Empirically}, Figure~\ref{fig:prelim}(a,c) shows that entropy exhibits weaker correlation with gradient variance and is consistently outperformed by geometry-aware measures under bootstrap \textit{(reflecting gradients not performance)}, supporting the anisotropic gap (Appendix~\ref{app:stats}).

This anisotropic gap also arises in token-level entropy, where entropy can be used to upper-bound gradient variance under a uniformity assumption. Intuitively, minimizing $H(x)$ reduces a worst-case upper bound on variance by treating all pairwise disagreements as equally severe. Formally,
\begin{equation}
    \mathcal{V}(x) \approx \frac{1}{2} \sum_{i=1}^K \sum_{j=1}^K \pi_i \pi_j \|\mu_i - \mu_j\|^2 
    \;\le\; \frac{\Delta_{\max}^2}{2} H(x),
    \label{eq:pairwise_var}
\end{equation}
where the variance is driven by two factors: the \emph{probability dispersion} ($\pi_i \pi_j$) and the \emph{disagreement magnitude} ($\|\mu_i - \mu_j\|^2$). The bound replaces all pairwise disagreements with a uniform worst-case constant $\Delta_{\max}^2$ (i.e., $\|\mu_i - \mu_j\|^2 \le \Delta_{\max}^2$), thereby collapsing heterogeneous geometric structure into a single scale. As a result, entropy provides only a loose, ineffective upper bound on the true variance when disagreement is anisotropic. Detailed derivations are provided in Appendix~\ref{app:math_proofs}.

\subsection{The Calibration Gap}

While the anisotropic gap highlights a structural limitation of entropy-based measures, a second challenge arises from their role in optimization dynamics. In GRPO, the update is driven by reward-induced contrast and the variance of this update can be expressed as $\mathrm{Var}(\hat{g}(x)) =  \mathrm{Var}\big(\hat{A}(Y)\, g(Y)\big)$, where $Y \sim \pi_\theta(\cdot \mid x)$. This expression reveals a key coupling: the variability of $\hat{A}(Y)$ is governed by reward variance within the rollout group, while the variability of $g(Y)$ reflects how different responses induce different update directions. When rollouts are highly concentrated, rewards collapse around their average and $\hat{A}(Y)$ vanishes, leading to weak or zero updates. In contrast, when rollouts are extremely dispersed, both $\hat{A}(Y)$ and $g(Y)$ exhibit large variability, increasing gradient signals but also potentially increasing noise and destabilizing optimization. However, entropy-based methods such as $H_{\mathrm{sem}}(x)$ provide no information about how informative the sampled responses are for learning. In particular, inputs with similar uncertainty can induce fundamentally different optimization behaviors depending on the dispersion of their rewards, which directly determines the scale of the normalized advantage $\hat{A}(Y)$. \textbf{Empirically}, Figure~\ref{fig:prelim}(b) shows entropy and semantic entropy are less effective at retrieving high-variance samples, indicating misalignment with reward-driven learning signals. Moreover, Figure~\ref{fig:prelim}(d) shows weaker held-out alignment for entropy-based measures, indicating that this misalignment with reward-driven variability generalizes beyond observed samples. (Appendix~\ref{app:stats}).

Therefore, when existing methods solely suppress the advantage of high uncertainty tasks \cite{zhang2025right, chen2025grpo}, they implicitly suppress the variety of responses and potentially reward variance, further suppressing the tasks which provides rich learning signals: 
\begin{equation}
\mathbb{E}_{x}\big[\mathrm{Var}(\hat{A}(Y) \mid x)\big]
\;\downarrow
\quad \Longrightarrow \quad
\mathbb{E}_{x}\big[\|\hat{g}(x)\|\big]
\;\downarrow,
\end{equation}
resulting in weaker updates and diminished learning signals. More broadly, this induces a biased training distribution. Let $\mathcal{D}_H := \{x : H_{\mathrm{sem}}(x) \le \tau\}$ denote the subset retained after entropy filtering. Since high-entropy inputs are more likely to contain diverse and reward-differentiated responses, removing them reduces the overall gradient variability:
\begin{equation}
\mathbb{E}_{x \sim \mathcal{D}_H} \big[ \mathrm{Var}(\hat{g}(x)) \big]
\;\ll\;
\mathbb{E}_{x \sim \mathcal{D}} \big[ \mathrm{Var}(\hat{g}(x)) \big],
\end{equation}
thereby suppressing the primary source of corrective learning signals in reasoning tasks. Formally:

\begin{proposition}
\textbf{The Calibration Gap.}
SE-type measures decouple uncertainty from reward, and thus fail to capture reward-aligned uncertainty. As a result, entropy-based suppression systematically removes high-variance, high-signal samples, inducing a biased optimization process that underestimates gradient variance and weakens policy improvement.
\end{proposition}

\textbf{Intuitively,} entropy-based methods treats all dispersion as noise, while in RL optimization, dispersion often encodes \emph{learning opportunity}. High-entropy inputs frequently correspond to ambiguous or partially solved problems where reward signals differentiate competing reasoning paths. By ignoring this alignment, $H_{\mathrm{sem}}$ conflates \emph{uninformative uncertainty} with \emph{useful disagreement}, leading to mis-calibrated uncertainty estimates and suboptimal training dynamics. This calibration gap is particularly pronounced in reasoning settings, where meaningful progress relies on precisely those high-variance trajectories that entropy-based methods tend to suppress.

\begin{figure*}[t]
	\centering
	\includegraphics[width=1\linewidth]{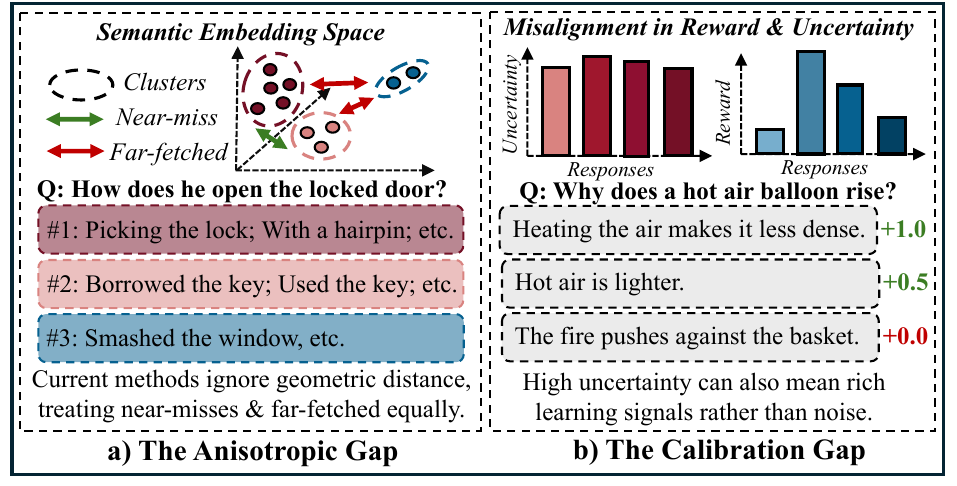}
        \vspace{-20pt}
	\caption{\textbf{The Two Core Gaps: An Intuitive View and Demonstration.} Details in Section ~\ref{sec:preliminary}.
    }
        \vspace{-15pt}
    \label{fig:gaps}
\end{figure*}

\section{Geometric-aware Calibrated Policy Optimization}
\label{sec:method}

As discussed in the previous section, our preliminary analysis identifies two key limitations of entropy-based uncertainty: 1) \textbf{\emph{the anisotropic gap}}, where entropy captures only coarse probability dispersion while ignoring the geometric magnitude of disagreement, and 2) \textbf{\emph{the calibration gap}}, where entropy is decoupled from reward informativeness and may suppress useful learning signals. 

To address these limitations, we propose \textbf{G}eometric-aware reward-\textbf{C}alibrated \textbf{P}olicy \textbf{O}ptimization \textbf{(GCPO)}, a novel framework that introduces a query-level modulation mechanism that adjusts the strength of policy updates based on both semantic structure and reward signals. Concretely, to overcome the anisotropic gap, we develop a family of geometry-aware uncertainty measures to capture semantic disagreement beyond entropy, i.e., \textbf{\emph{Cosine Dispersion}} (section~\ref{sec:cd}) and \textbf{\emph{Barycentric Transport}} (section~\ref{sec:bot}), and further to overcome the calibration gap, we incorporate \textbf{\emph{Reward Dispersion}} (section~\ref{sec:rd}) to align uncertainty with learning signal strength. Together, these components provide a structured and calibrated mechanism for controlling policy updates in reasoning tasks.

\subsection{Cosine Dispersion: Geometry-aware Disagreement}
\label{sec:cd}

As discussed in Section~\ref{sec:a_gap}, entropy-based measures depend solely on probability mass and therefore fail to capture the \emph{magnitude} of disagreement between responses. To address this anisotropic gap, we first introduce this simple yet effective method, \emph{Cosine Dispersion (CD)}, which measures semantic disagreement directly in embedding space. Given rollout answers $\{a_i\}_{i=1}^G$, we compute normalized embeddings $\mu_i$. The pairwise cosine distance is defined as $D_{ij} = \operatorname{clip}(1-\mu_i^\top \mu_j, 0, 1)$, and with uniform weights $\pi_i = 1/G$, CD is formally defined as:
\begin{equation}
\operatorname{CD}(q)
=
\sum_{i,j} \pi_i \pi_j D_{ij}
=
\pi^\top D \pi.
\end{equation}

Unlike entropy, CD reflects not only whether responses differ, but also \emph{how far apart} they are in semantic space. When responses are aligned, CD remains small; when disagreement becomes geometrically significant, CD increases accordingly. We translate this into a reliability weight:
\begin{equation}
\omega_{\mathrm{CD}}(q)
=
\operatorname{clip}
\left(
1-\alpha_G \operatorname{CD}(q)^2,
0,1
\right).
\end{equation}

This formulation provides a direct geometric correction to entropy-based uncertainty by penalizing large semantic deviations while remaining tolerant to minor variations. In practice, CD is particularly effective when semantic variation is smooth and does not exhibit clear multi-modal structure.

\begin{figure*}[t]
	\centering
	\includegraphics[width=1\linewidth]{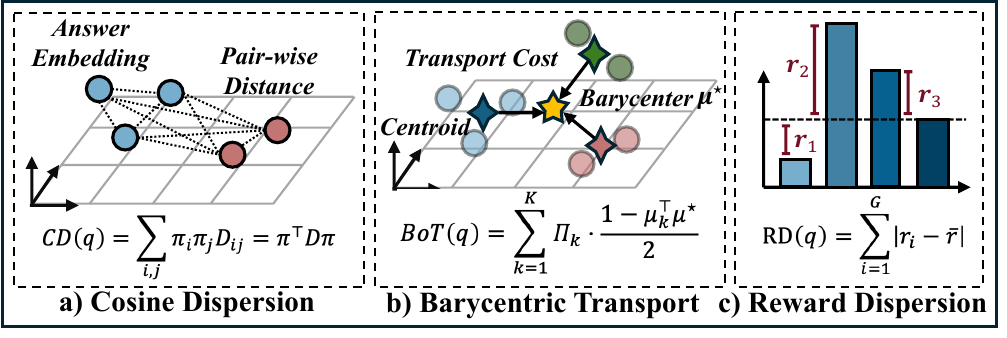}
        \vspace{-20pt}
	\caption{\textbf{Illustration of the key components of GCPO.} The formulas above highlight the core principles rather than the full complete method; complete details are provided in Section~\ref{sec:method}.
    }
        \vspace{-15pt}
    \label{fig:frame}
\end{figure*}

\subsection{Barycentric Transport: Mode-level Semantic Structure}
\label{sec:bot}

While CD captures pairwise geometric disagreement, it treats all samples independently and does not explicitly account for how probability mass is distributed across distinct semantic modes. This limitation becomes pronounced when rollouts form multiple coherent semantic clusters, where within-cluster variation is benign but across-cluster divergence is substantial.

To address this, we draw inspiration from \emph{optimal transport} (OT), which studies how probability mass is redistributed between distributions with minimal cost. In our setting, the rollout group induces an empirical distribution over semantic representations, and uncertainty can be interpreted as the cost of transporting this distribution toward a consensus representation. Formally, given distributions $\{\mu_k\}_{k=1}^K$ with weights $\{\Pi_k\}_{k=1}^K$, the Wasserstein barycenter is defined as
$\mu^\star
=
\arg\min_{\mu}
\sum_{k=1}^{K} \Pi_k \, \mathcal{W}_c(\mu_k, \mu)$,
where $\mathcal{W}_c(\cdot,\cdot)$ denotes the optimal transport cost under ground metric $c(\cdot,\cdot)$, and $\Pi_k$ is the probability mass of cluster $k$. Intuitively, $\mu^\star$ represents a consensus distribution that minimizes the total transport cost from all semantic modes. Computing exact OT distances is computationally expensive and unnecessary in our setting. We therefore adopt a tractable approximation via \emph{Barycentric Transport (BoT)}, which measures the deviation of semantic modes from a shared barycentric direction in embedding space. Specifically, we first group answers into clusters $\{1,\dots,K\}$ using an entailment-based procedure. Let $\Pi_k = \sum_{i:z_i=k} \pi_i$ denote the cluster mass, where $\pi_i = 1/G$ is the rollout weight, and let $\mu_k$ denote the normalized centroid of cluster $k$. We approximate the barycenter direction as
\begin{equation}
\mu^\star = \frac{\sum_k \Pi_k \mu_k}{\left\|\sum_k \Pi_k \mu_k\right\|_2},
\end{equation}
which serves as a first-order surrogate of the OT barycenter in the embedding space. We then define the Barycentric Transport score as

\vspace{-10pt}
\begin{equation}
\operatorname{BoT}(q)
=
\sum_{k=1}^{K}
\Pi_k \cdot \frac{1-\mu_k^\top \mu^\star}{2},
\end{equation}

\vspace{-10pt}

where $\frac{1-\mu_k^\top \mu^\star}{2}$ corresponds to the cosine-based transport cost between cluster $k$ and the barycenter. This quantity can be interpreted as a first-order proxy of the total transport cost required to align all semantic modes with the consensus representation. Unlike CD, BoT collapses intra-cluster variation and focuses on inter-mode structure, making it robust to paraphrastic redundancy while remaining sensitive to true semantic divergence. The corresponding weight is:$\omega_{\mathrm{BoT}}(q)
=
\operatorname{clip}
\left(
1-\alpha_G \operatorname{BoT}(q)^2,
0,1
\right)$.

CD and BoT therefore represent two alternative instantiations of geometry-aware uncertainty. CD captures fine-grained pairwise disagreement and is effective when variation is smooth and continuous, while BoT captures structured dispersion and is more suitable when in distinct semantic clusters. We further discuss their theoretical connections with gradient variance in Appendix~\ref{app:link}. Also, as we discussed above, this connection relies on a mild semantic-gradient alignment assumption; we discuss its implications and empirical validation in Appendix~\ref{app:assumption}.

\subsection{Reward Dispersion: Calibration via Informative Signal}
\label{sec:rd}

While geometry-aware measures address the anisotropic gap, they do not resolve the calibration gap: semantic disagreement alone does not indicate whether a rollout group provides useful learning signals. As discussed in Section~\ref{sec:preliminary}, GRPO updates are driven by reward contrast, suggesting that uncertainty should be calibrated with respect to reward variability. We define \emph{Reward Dispersion (RD)} to quantify reward informativeness. Let $\bar{r} = \frac{1}{G}\sum_i r_i$. Then
\begin{equation}
\operatorname{RD}_{\mathrm{raw}}(q)
=
\sum_{i=1}^{G} |r_i - \bar{r}|,
\quad
\operatorname{RD}(q)
=
\operatorname{clip}
\left(
\frac{\operatorname{RD}_{\mathrm{raw}}(q)}{\operatorname{RD}_{\max}(G)},
0,1
\right),
\quad
\omega_{\mathrm{RD}}(q)
=
1+\alpha_G \operatorname{RD}(q).
\end{equation}

In contrast to geometric terms, RD amplifies updates when reward contrast is strong. This reflects the intuition that rollout groups with higher reward variability provide more informative preference signals, whereas uniformly scored rollouts contribute limited learning signal. In this way, RD provides a mechanism to align uncertainty with optimization dynamics. We discuss the interaction between uncertainty and reward dispersion, and safeguards in Appendix~\ref{app:rd_discussion}.

\vspace{-5pt}
\subsection{Unified Modulation of Policy Updates}

We integrate the above components into a unified modulation of GRPO updates:
$\widetilde{A}_i
=
A_i \cdot \omega_{\mathrm{geo}}(q) \cdot \omega_{\mathrm{RD}}(q)$,
where $\omega_{\mathrm{geo}}$ is instantiated as either $\omega_{\mathrm{CD}}$ or $\omega_{\mathrm{BoT}}$. Rather than a heuristic rescaling, this formulation can be interpreted as a structured control mechanism over policy updates. The geometric term evaluates the \emph{reliability of semantic agreement}, while the reward term evaluates the \emph{informativeness of the rollout group}. Their interaction determines the effective strength of gradient updates, balancing stability and learning efficiency. The additional computation is negligible compared to rollout generation (see Appendix~\ref{app:complexity}). For answer representation, each rollout $y_i$ is mapped to an answer $a_i = \operatorname{Extract}(y_i)$ and embedded as $\mu_i = f_{\mathrm{emb}}(a_i)/\|f_{\mathrm{emb}}(a_i)\|_2$. We adopt a group-size normalization $\alpha_G = \alpha_{\mathrm{base}}/\log G$ to ensure stable scaling. We provide further discussion on the underlying assumptions, approximation properties, optimization effects, and practical considerations of the proposed framework in Appendix~\ref{app:discussion}.

\begin{table*}[t]
\renewcommand{\arraystretch}{1.2}
\centering
\resizebox{\textwidth}{!}{
\begin{tabular}{c ccc ccc cc}
\toprule
\multirow{2}{*}{\textbf{Method}} &
\multicolumn{3}{c}{\textbf{NarrativeQA}} &
\multicolumn{3}{c}{\textbf{Qasper}} &
\multicolumn{2}{c}{\textbf{HotpotQA}} \\
\cmidrule(lr){2-4}\cmidrule(lr){5-7}\cmidrule(lr){8-9}
& \textbf{Acc} & \textbf{BLEU} & \textbf{F1}
& \textbf{Acc} & \textbf{BLEU} & \textbf{F1}
& \textbf{F1} & \textbf{EM} \\
\midrule\midrule

GRPO
& 66.23{\tiny$\pm$1.53} & 48.71{\tiny$\pm$1.70} & 67.02{\tiny$\pm$0.95}
& 11.59{\tiny$\pm$1.63} & 9.40{\tiny$\pm$0.95} & 20.51{\tiny$\pm$1.36}
& 50.00{\tiny$\pm$2.92} & 37.06{\tiny$\pm$2.69} \\

\midrule

w/t SE
& 71.76{\tiny$\pm$0.47} & 52.20{\tiny$\pm$1.80} & 72.24{\tiny$\pm$0.35}
& 13.29{\tiny$\pm$0.40} & 10.80{\tiny$\pm$1.37} & 22.23{\tiny$\pm$0.47}
& 57.33{\tiny$\pm$1.29} & 44.18{\tiny$\pm$1.13} \\

w/t KLE
& 72.03{\tiny$\pm$0.50} & 52.11{\tiny$\pm$1.40} & 72.39{\tiny$\pm$0.30}
& 13.46{\tiny$\pm$0.63} & \textbf{11.62{\tiny$\pm$0.35}} & 22.26{\tiny$\pm$0.30}
& 57.62{\tiny$\pm$0.90} & 44.31{\tiny$\pm$0.83} \\

EMPO
& 57.38{\tiny$\pm$0.25} & 41.32{\tiny$\pm$0.28} & 60.96{\tiny$\pm$0.20}
& 12.00{\tiny$\pm$0.27} & 9.58{\tiny$\pm$0.17} & 18.05{\tiny$\pm$0.23}
& { n/a } & { n/a } \\

CARE
& 71.40{\tiny$\pm$0.60} & 51.84{\tiny$\pm$0.65} & 71.65{\tiny$\pm$0.48}
& 13.32{\tiny$\pm$0.75} & 10.83{\tiny$\pm$0.30} & 22.07{\tiny$\pm$0.46}
& 57.03{\tiny$\pm$1.99} & 43.75{\tiny$\pm$2.14} \\

\midrule

\rowcolor{gray!15}
\textbf{w/t CD}
& \underline{72.59{\tiny$\pm$0.48}} & \underline{53.72{\tiny$\pm$0.89}} & \underline{73.25{\tiny$\pm$0.27}}
& \underline{13.67{\tiny$\pm$0.33}} & \underline{11.28{\tiny$\pm$0.83}} & \underline{24.25{\tiny$\pm$0.17}}
& \underline{59.97{\tiny$\pm$0.67}} & \underline{46.80{\tiny$\pm$0.77}} \\

\textit{\small $\Delta$ vs.GRPO}
& \textcolor{green!50!black}{\small$\uparrow$9.60\%}
& \textcolor{green!50!black}{\small$\uparrow$10.29\%}
& \textcolor{green!50!black}{\small$\uparrow$9.30\%}
& \textcolor{green!50!black}{\small$\uparrow$17.95\%}
& \textcolor{green!50!black}{\small$\uparrow$20.00\%}
& \textcolor{green!50!black}{\small$\uparrow$18.24\%}
& \textcolor{green!50!black}{\small$\uparrow$19.94\%}
& \textcolor{green!50!black}{\small$\uparrow$26.28\%} \\

\rowcolor{gray!15}
\textbf{w/t BoT}
& \textbf{73.77}{\tiny$\pm$0.16} & \textbf{54.85}{\tiny$\pm$0.74} & \textbf{73.82}{\tiny$\pm$0.07}
& \textbf{13.91}{\tiny$\pm$0.63} & 10.94{\tiny$\pm$1.49} & \textbf{24.37}{\tiny$\pm$0.64}
& \textbf{60.78}{\tiny$\pm$1.06} & \textbf{47.27}{\tiny$\pm$1.08} \\

\textit{\small $\Delta$ vs.GRPO}
& \textcolor{green!50!black}{\small$\uparrow$11.38\%}
& \textcolor{green!50!black}{\small$\uparrow$12.61\%}
& \textcolor{green!50!black}{\small$\uparrow$10.15\%}
& \textcolor{green!50!black}{\small$\uparrow$20.02\%}
& \textcolor{green!50!black}{\small$\uparrow$16.38\%}
& \textcolor{green!50!black}{\small$\uparrow$18.82\%}
& \textcolor{green!50!black}{\small$\uparrow$21.56\%}
& \textcolor{green!50!black}{\small$\uparrow$27.55\%} \\

\midrule
\bottomrule
\end{tabular}
}
\vspace{-5pt}
\caption{Performance comparison on NarrativeQA, Qasper, and HotpotQA on Qwen2.5-1.5b-instruct. We report mean and standard deviation, and improvement over raw GRPO across runs.}
\label{tab:main_experiment}
\vspace{-15pt}
\end{table*}

\section{Experiments}
\subsection{Experiment Setup}
\noindent \textbf{Benchmarks. }We evaluate GCPO on a diverse set of benchmarks spanning open-ended QA, fact-driven QA, and mathematical reasoning. We first focus on open-ended question answering datasets, including NarrativeQA \cite{kovcisky2018narrativeqa} and Qasper \cite{dasigi2021dataset}, where multiple semantically distinct yet partially correct responses naturally arise. Such settings induce high response-level variability, making them particularly suitable for studying uncertainty signals and their interaction with optimization dynamics. We report BLEU, token-level F1, and accuracy (F1 > 0.5) for evaluation. We further include fact-driven QA via HotpotQA \cite{yang2018hotpotqa}, which emphasizes multi-hop reasoning with more constrained answer spaces. 
Evaluation follows standard token-level F1 and Exact Match (EM). Finally, We additionally evaluate on mathematical reasoning benchmarks, including MATH500 \cite{hendrycks2021measuring}, AIME24, and OlympiadBench \cite{huang2024olympicarena}, to examine generalization to structured reasoning tasks with deterministic answers. Detailed dataset descriptions and configurations are provided in Appendix~\ref{app:benchmarks}.

\begin{wraptable}{r}{0.50\textwidth}
\renewcommand{\arraystretch}{1.1}
\vspace{-8pt}
\centering
\setlength{\tabcolsep}{1pt}
\footnotesize
\begin{tabular}{c ccc}
\toprule
\textbf{Method} & \textbf{MATH500} & \textbf{AIME24} & \textbf{OlympiadBench} \\
\midrule\midrule

Qwen2.5-Math-1.5B
& 61.8 & \underline{16.7} & 28.4 \\

\midrule

GRPO + SE
& 68.53{\tiny$\pm$0.99}
& 15.57{\tiny$\pm$1.60}
& 36.27{\tiny$\pm$1.39} \\

*\textit{Seed-GRPO}
& \textbf{75.4}
& \textbf{23.3}
& \textbf{41.3} \\

GRPO + KLE
& 67.50{\tiny$\pm$1.41}
& 12.20{\tiny$\pm$1.91}
& 33.33{\tiny$\pm$2.52} \\

EMPO
& 73.0
& 13.3
& 36.6 \\

GRPO-Care
& 68.07{\tiny$\pm$0.46}
& 14.43{\tiny$\pm$1.96}
& 35.78{\tiny$\pm$1.83} \\

\midrule

\rowcolor{gray!15}
\textbf{GRPO + CD}
& \underline{74.67{\tiny$\pm$0.12}}
& \underline{17.80{\tiny$\pm$1.56}}
& \underline{38.24{\tiny$\pm$1.20}} \\

\rowcolor{gray!15}
\textbf{GRPO + BoT}
& 73.80{\tiny$\pm$0.53}
& 15.57{\tiny$\pm$1.60}
& 37.25{\tiny$\pm$0.69} \\

\midrule
\bottomrule
\end{tabular}
\vspace{-8pt}
\caption{Performance on math reasoning benchmarks, where the gain becomes less salient.}
\label{tab:math_benchmarks}
\vspace{-10pt}
\end{wraptable}


\noindent \textbf{Baselines. }
Our objective is not to maximize absolute task performance, but to evaluate whether response-level uncertainty signals meaningfully improve GRPO-style optimization by addressing the limitations identified in Section~\ref{sec:preliminary}. We therefore focus on baselines directly aligned with this scope, including: (i) \textbf{GRPO}, without uncertainty modulation \cite{shao2024deepseekmath}; (ii) \textbf{semantic entropy-based methods}, such as Seed-GRPO \cite{chen2025seed} and EMPO \cite{zhang2025right}, which regulate updates using response-level entropy; (iii) \textbf{kernel-based entropy}, which refines uncertainty via similarity-aware measures \cite{nikitin2024kernel}; and (iv) \textbf{consistency-based methods}, e.g., GRPO-CARE \cite{chen2025grpo}, which use agreement across rollouts as a proxy for uncertainty. To provide broader context, \textbf{\emph{we additionally adapt several most recent baselines}} in Appendix~\ref{app:additional} for a broader comparison, although they are less directly aligned with our primary focus. All methods are evaluated under a unified protocol with the same backbone (Qwen-2.5-1.5B-Instruct) and configuration. Implementation details and adaptation procedures are provided in Appendix~\ref{app:baselines}.

\subsection{Main Experiment}
\label{sec:main_experiment}

We present the main results of GCPO in Table~\ref{tab:main_experiment}. Overall, we can see GCPO consistently outperforms all baselines, demonstrating that uncertainty signals that properly align with gradient variance as we discussed in the analysis, leading to more effective GRPO optimization. \textbf{First, regulating uncertainty is essential and effective for GRPO.} All uncertainty-based methods substantially outperform raw GRPO with 10\% improvement on NarrativeQA and 20\% improvement on Qasper and HotpotQA. This large and consistent gap indicates that one primary limitation of GRPO is the lack of mechanisms to differentiate informative from uninformative rollout groups. Uncertainty signals effectively reweight gradient contributions, reshaping the training dynamics. 
\begin{wraptable}{r}{0.48\textwidth}
\centering
\small
\renewcommand{\arraystretch}{1.05}
\setlength{\tabcolsep}{2pt}
\begin{tabular}{ccccc}
\toprule

\textbf{w/t Success Rate} & \textbf{Entropy} & \textbf{SE} & \textbf{CD} & \textbf{BoT} \\

\midrule\midrule
Math (Rho)
& -0.317 & -0.733 & -0.660 & -0.728 \\

\rowcolor{gray!15}
{\textit{p}}
& {\scriptsize\textit{<0.001}} & {\scriptsize\textit{<0.001}} & {\scriptsize\textit{<0.001}} & {\scriptsize\textit{<0.001}} \\

\midrule
Olympiad (Rho)
& -0.424 & -0.435 & -0.467 & -0.452 \\

\rowcolor{gray!15}
{\textit{p}}
& {\scriptsize\textit{<0.001}} & {\scriptsize\textit{<0.001}} & {\scriptsize\textit{<0.001}} & {\scriptsize\textit{<0.001}} \\

\midrule
\bottomrule
\end{tabular}

\vspace{-5pt}
\caption{Correlation with Success Rate.}
\vspace{-10pt}
\label{tab:success}
\end{wraptable}

\begin{wraptable}{r}{0.48\textwidth}
\centering
\small
\renewcommand{\arraystretch}{1.05}
\setlength{\tabcolsep}{2pt}
\vspace{-6pt}
\begin{tabular}{ccccc}

\toprule

\textbf{w/t Gradient Variance} & \textbf{Entropy} & \textbf{SE} & \textbf{CD} & \textbf{BoT} \\

\midrule\midrule

Math (Rho)
& 0.285 & 0.406 & 0.578 & 0.610 \\
\rowcolor{gray!15}

{\textit{p}}
& {\scriptsize\textit{<0.001}} & {\scriptsize\textit{<0.001}} & {\scriptsize\textit{<0.001}} & {\scriptsize\textit{<0.001}} \\

\midrule
Olympiad (Rho)
& -0.066 & 0.165 & 0.136 & 0.147 \\

\rowcolor{gray!15}
{\textit{p}}
& {\scriptsize\textit{>0.05}} & {\scriptsize\textit{<0.001}} & {\scriptsize\textit{<0.05}} & {\scriptsize\textit{<0.05}} \\

\midrule
\bottomrule

\end{tabular}

\vspace{-6pt}
\caption{Correlation with Gradient Variance.}
\vspace{-10pt}
\label{tab:gradient}

\end{wraptable}

\textbf{Second, performance gains reflect alignment with informative gradient variance.} Improvements are most pronounced on NarrativeQA, where GCPO achieves substantially higher F1 over existing uncertainty-aware methods, while gains on Qasper are smaller.
This mirrors our analysis: NarrativeQA exhibits stronger coupling between response diversity and reward differentiation, making gradient variance more informative, whereas Qasper contains high semantic diversity but weaker reward contrast, echoing our previous analysis that the benefit of improved uncertainty estimation depends on how well it captures \emph{informative} variance rather than dispersion alone. 

\textbf{Third, uncertainty regulation improves training stability through variance control.} The standard deviations across all uncertainty-aware methods compared to raw GRPO are consistently reduced. Moreover, GCPO further improves stability while maintaining performance gains, suggesting that geometry-aware and calibrated signals suppress spurious variance without eliminating informative gradient diversity. \textbf{Overall,} the results suggest that uncertainty should be viewed as a mechanism for \emph{calibrating gradient variance with respect to learning signal quality}, rather than simply boosting or suppressing it.


\begin{wrapfigure}{r}{0.5\textwidth}
\vspace{-10pt}
\centering
\includegraphics[width=\linewidth]{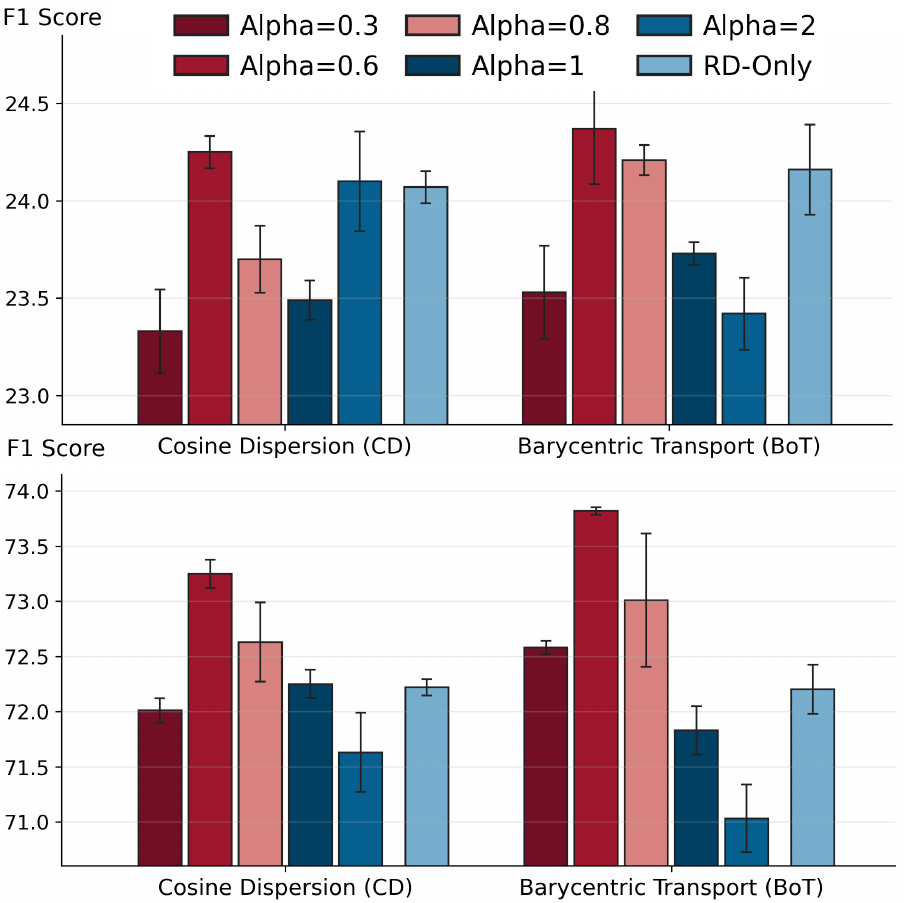}
\vspace{-20pt}
\caption{Effect of $\alpha$ and RD on NarrativeQA.}
\label{fig:ablation}
\vspace{-10pt}
\end{wrapfigure}

\subsection{Experiment on Math Reasoning}
\label{sec:math}
We further evaluate GCPO and baseline methods on mathematical reasoning benchmarks, where we consider the full reasoning trajectories rather than only the final boxed answers. As shown in Table~\ref{tab:math_benchmarks}, the performance gains are less pronounced compared to QA benchmarks. While GCPO consistently outperforms traditional uncertainty-aware methods, it is less competitive with more comprehensive approaches. Notably, \textit{GRPO+SE} implements only the core semantic entropy mechanism from Seed-GRPO instead of a full implementation. To better understand this discrepancy, we conduct statistical analyses. As shown in Tables~\ref{tab:success} and~\ref{tab:gradient}, both benchmarks exhibit strong correlations between uncertainty and success rate. However, the correlation between uncertainty and gradient variance is notably weaker on OlympiadBench compared to MATH datasets. This divergence helps explain why GCPO yields clear improvements on MATH500 but more limited gains on AIME24 and OlympiadBench. Intuitively, mathematical reasoning tasks often admit multiple semantically diverse yet valid solution paths, while reasoning trajectories may remain nearly identical until a single critical error determines correctness. In such settings, response-level uncertainty becomes less aligned with informative gradient variance, limiting its effectiveness as a control signal, showing the importance of matching uncertainty signal design to task characteristics to ensure effective optimization.

\begin{wrapfigure}{r}{0.60\textwidth}

\vspace{-10pt}
\centering
\includegraphics[width=\linewidth]{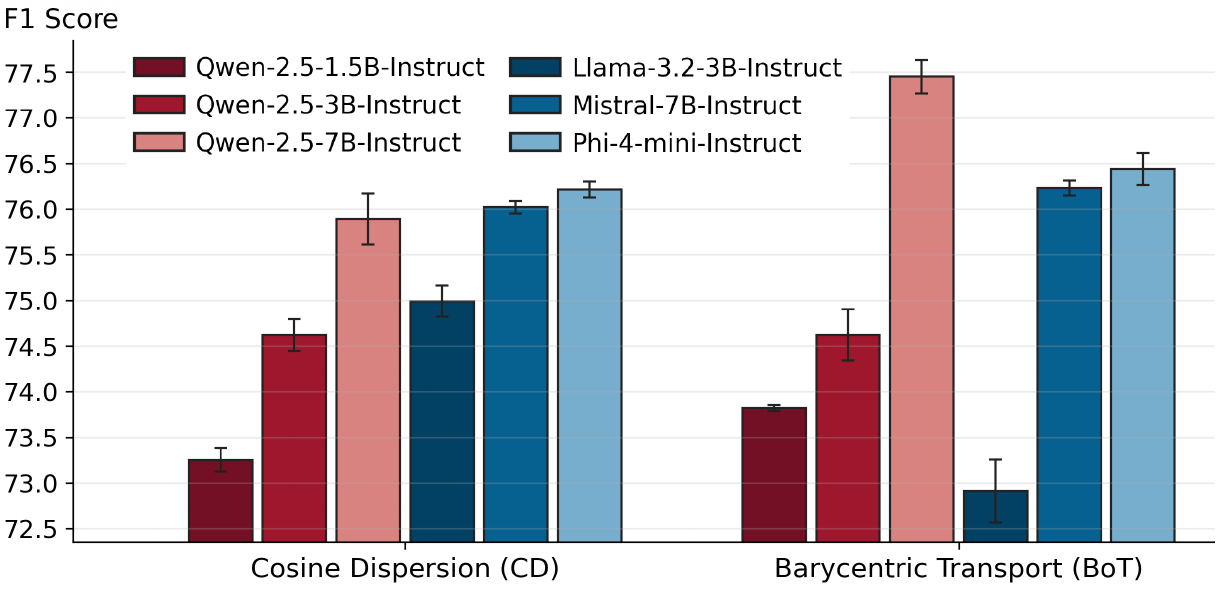}
\vspace{-20pt}
\caption{Effect of $\alpha$ on Qasper.}
\label{fig:size}
\vspace{-10pt}
\end{wrapfigure}
\subsection{Ablation Studies}

We conduct ablations on the scaling factor $\alpha$, reward dispersion (RD), and model capacity. Moderate $\alpha$ values (e.g., $\alpha=0.6$) perform best, while too small or large values degrade performance, indicating that geometry-aware signals are most effective as calibrated modulation. RD-only remains competitive but consistently underperforms CD and BoT, suggesting that reward variation alone is insufficient without structural alignment. Across tasks, RD correlates with gradient informativeness but fails to capture semantic disagreement, whereas CD and BoT provide more consistent gains by modeling response structure. Finally, improvements scale with model size and remain stable across architectures (Figure~\ref{fig:size}), indicating that GCPO is complementary to model scaling. Overall, effective uncertainty-aware optimization requires both calibrated signal strength and alignment with the underlying geometry of the response space.

\vspace{-5pt}
\section{Conclusion}
\vspace{-5pt}

We present GCPO, a principled framework that reinterprets uncertainty as a mechanism for regulating gradient variance in post-training. Through theoretical and empirical analysis, we identify fundamental limitations of entropy-based measures and show that effective uncertainty signals must capture both geometric disagreement and reward-aligned informativeness. By integrating geometry-aware structure with reward calibration, GCPO achieves more reliable alignment with optimization dynamics and consistently improves performance across benchmarks. More broadly, our results suggest that uncertainty should be treated not as a generic proxy for dispersion, but as a structured signal tied to learning dynamics, highlighting the importance of aligning uncertainty estimation with task characteristics and optimization objectives.

\newpage
\bibliography{reference}
\bibliographystyle{unsrtnat}

\newpage
\appendix

\section{Related Work}

\subsection{Uncertainty Estimation for Generation and Reasoning.}
Large language models (LLMs) have advanced rapidly in recent years \cite{ye2025llms4all, xiong2025deliberate, sanchez2022effects, xiong2025enhancing, xiong2026scaling, li2024cheffusion, li2025adaptive, li2025crochetbench}. Building on this progress, LLM-driven agents have gained prominence for their ability to plan, interact, and solve complex tasks with limited human oversight \cite{zhang2026mapro, zhang2025agentrouter, shi2026ng, huang2026evolverouter, huang2026glen, bao2026drift}. A central practical challenge accompanying these successes is hallucination: models confidently producing incorrect or fabricated facts. This leads to the growing research of reliably quantifying and leveraging model uncertainty to faithfully represents model capabilities to recognize and mitigate hallucinations in trustworthy AI. 

A primary line of work estimates uncertainty directly from model predictions, where token-level entropy is used as a proxy for confidence. These methods interpret uncertainty as dispersion over next-token distributions and have been widely used in early studies \cite{fadeeva2023lm, fadeeva2024fact}. More recent approaches incorporate such signals into training, for example by using entropy to modulate gradient magnitude while preserving update direction \cite{hu2026entropy}. While effective for capturing local prediction uncertainty, these approaches remain fundamentally limited: token-level entropy reflects lexical ambiguity, but does not capture higher-level semantic structure across complete responses. As a result, it is often insufficient for reasoning-intensive tasks where correctness depends on global coherence rather than local token uncertainty.

To address this limitation, a growing body of work shifts to \emph{response-level uncertainty}, estimating confidence from agreement across multiple generated samples. Early approaches quantify agreement through exact-match frequency or consistency scores \cite{xiong2023can}, while later work extends this idea to semantic and factual consistency using external evaluators or LLM-based judgments \cite{Manakul2023SelfCheckGPTZB, Zhang2024LUQLU, Jiang2024GraphbasedUM}. Entropy-based formulations provide a more principled probabilistic interpretation by modeling distributions over semantic outcomes, including clustering-based semantic entropy \cite{kuhn2023semantic}, its discrete variant \cite{farquhar2024detecting}, and similarity-based extensions such as kernel entropy \cite{nikitin2024kernel} and nearest-neighbor entropy (Semantic Nearest Neighbor Entropy). 

However, as discussed in Section~\ref{sec:preliminary}, these methods exhibit two key limitations. First, they characterize the distribution of responses but ignore the magnitude and structure of disagreement, treating heterogeneous semantic differences uniformly (the anisotropic gap). Second, they are largely decoupled from reward signals, failing to distinguish between informative diversity and spurious variability (the calibration gap). These limitations motivate our geometry-aware and reward-calibrated formulation.

\subsection{Variance Control in Policy Optimization (Complementary Perspectives).}
Beyond uncertainty estimation, a separate line of work studies stability in policy optimization through the lens of variance control. These approaches operate at different levels of the optimization pipeline and are largely orthogonal to our focus. At the objective level, KL-constrained and divergence-based methods derive stable updates through analytical weighting or regularization \cite{zhang2025gvpo,han2026survey}. Other methods address gradient instability via conflict resolution or reformulation, including probabilistic gradient arbitration \cite{qiang2026plasticity} and contrastive group optimization \cite{zhang2026rence}. At finer granularity, sample-level approaches reweight updates using uncertainty or reward variability \cite{xie2026q}. Additional work targets complementary sources of variance, such as policy-ratio instability \cite{luo2026ratio}, diversity-driven reward shaping \cite{wei2026mmr}, gradient-geometry-based exploration \cite{liang2025can}, or step-level credit assignment \cite{parthasarathi2025grpo}. 

\textbf{It is important to notice that while these methods provide principled mechanisms for improving optimization stability, and both can be phrased as \emph{uncertainty-aware}, they focus on fundamentally different aspects of policy optimization.} These methods focus on policy-level or local variance sources and do not explicitly model how semantic structure and reward differentiation jointly influence gradient variance at the response level. In contrast, our approach targets this intermediate level by treating uncertainty as a structured signal that aligns semantic disagreement with reward informativeness to calibrate gradient contributions across rollout groups.

\section{Proof: Relations Between Entropy and Variance}
\label{app:math_proofs}

In Section \ref{sec:a_gap}, we established that minimizing entropy can act as a proxy for controlling Gradient Variance. Here, we provide the detailed mathematical derivations supporting Proposition \ref{prop:a_gap}. Let $\Pi = \{\pi_1, \dots, \pi_K\}$ be a probability distribution over $K$ tokens such that $\sum_{k=1}^K \pi_k = 1$ and $\pi_k \ge 0$.

\subsection{Derivation of the Variance Upper Bound}
\label{app:variance_bound_derivation}
First, we derive how the pairwise variance decomposition relates to the Gini Impurity under a bounded distance assumption.

Recall that the total gradient variance can be decomposed via the law of total variance \cite{casella2024statistical}. In practice, we focus on the component induced by differences across tokens, which can be approximated by the pairwise form in Eq.~\ref{eq:pairwise_var}:
\begin{equation}
    \mathcal{V}(q) \approx \frac{1}{2} \sum_{i=1}^K \sum_{j=1}^K \pi_i \pi_j \|\mu_i - \mu_j\|^2
\end{equation}

\textbf{Assumption:} Let $\Delta_{\max}$ be the maximum possible gradient disagreement between any two tokens, such that $\|\mu_i - \mu_j\|^2 \le \Delta_{\max}^2$ for all $i \neq j$. Note that if $i=j$, $\|\mu_i - \mu_j\|^2 = 0$.

Substituting this inequality into the variance expression:
\begin{align}
    \mathcal{V}(q) &\le \frac{1}{2} \sum_{i \neq j} \pi_i \pi_j \Delta_{\max}^2 \\
    &= \frac{\Delta_{\max}^2}{2} \underbrace{\sum_{i \neq j} \pi_i \pi_j}_{\text{Sum of off-diagonal probs}}
\end{align}

Using the identity:
$$ (\sum_{i=1}^K \pi_i)^2 = \sum_{i=1}^K \pi_i^2 + \sum_{i \neq j} \pi_i \pi_j = 1, $$
we obtain:
$$ \sum_{i \neq j} \pi_i \pi_j = 1 - \sum_{i=1}^K \pi_i^2. $$

Substituting back:
\begin{equation}
    \mathcal{V}(q) \le \frac{\Delta_{\max}^2}{2} \left( 1 - \sum_{k=1}^K \pi_k^2 \right) = \frac{\Delta_{\max}^2}{2} G(\Pi),
\end{equation}
where $G(\Pi)$ is the Gini Impurity. This shows that the gradient variance is upper-bounded by Gini Impurity scaled by the worst-case geometric disagreement.

\subsection{Proof of Strict Monotonicity for Binary Cases ($K=2$)}
\textbf{Claim:} When $K=2$, there is a strict one-to-one monotonic mapping between Shannon Entropy and Gini Impurity.

\begin{proof}
Consider $\Pi = \{\pi, 1-\pi\}$ with $\pi \in [0, 0.5]$:
\begin{align}
    H(\pi) &= -\pi \ln \pi - (1-\pi) \ln (1-\pi), \\
    G(\pi) &= 1 - (\pi^2 + (1-\pi)^2) = 2\pi - 2\pi^2.
\end{align}

$$ \frac{dG}{d\pi} = 2 - 4\pi > 0, \quad
\frac{dH}{d\pi} = \ln\left(\frac{1-\pi}{\pi}\right) > 0 \quad \text{for } \pi < 0.5. $$

Thus both are strictly increasing, establishing a bijection.
\end{proof}

\subsection{Proof of the Upper Bound ($G \le H$)}
\textbf{Claim:} For any $K \ge 2$, Gini Impurity is upper-bounded by Shannon Entropy (for non-degenerate distributions).

\begin{proof}
Using $x(1-x) \le -x \ln x$ for $x \in [0,1]$:
$$ \sum_{k=1}^K \pi_k(1-\pi_k) \le \sum_{k=1}^K -\pi_k \ln \pi_k, $$
which yields
$$ G(\Pi) = 1 - \sum_{k=1}^K \pi_k^2 \le H(\Pi). $$

Combining with the previous result:
$$ \mathcal{V}(q) \le \frac{\Delta_{\max}^2}{2} G(\Pi) \le \frac{\Delta_{\max}^2}{2} H(\Pi). $$

Thus, minimizing token-level entropy reduces an upper bound on this component of gradient variance.
\end{proof}

\subsection{Detailed Analysis of the Anisotropic Gap}
\label{app:anisotropic_gap_analysis}

In Section \ref{sec:a_gap}, we argued that while entropy provides a valid upper bound, it may become a loose proxy when gradient disagreements are anisotropic. We quantify this via the \textit{bound slack}:
\begin{equation}
    \mathcal{S}(\Pi) = \frac{\Delta_{\max}^2}{2} G(\Pi) - \frac{1}{2} \sum_{i \neq j} \pi_i \pi_j \|\mu_i - \mu_j\|^2.
\end{equation}

Rewriting:
\begin{align}
    \mathcal{S}(\Pi)
    =
    \frac{1}{2} \sum_{i \neq j}
    \pi_i \pi_j
    \left( \Delta_{\max}^2 - \|\mu_i - \mu_j\|^2 \right).
\end{align}

\paragraph{Regime 1: Tight Bound.}
If $\|\mu_i - \mu_j\|^2 \approx \Delta_{\max}^2$, then $\mathcal{S}(\Pi) \approx 0$. The bound is tight.

\paragraph{Regime 2: Loose Bound.}
If $\|\mu_i - \mu_j\|^2 \approx 0$, then
$$ \mathcal{S}(\Pi) \approx \frac{\Delta_{\max}^2}{2} G(\Pi). $$
The bound becomes maximally loose.

\paragraph{Conclusion.}
Entropy minimizes an upper bound that conflates true variance and slack. In anisotropic regimes, optimization may focus on reducing slack rather than true instability, motivating geometry-aware alternatives.

\section{Link between Geometry-aware Uncertainty and Gradient Variance.}
\label{app:link}
While entropy controls only the probability dispersion over response modes, CD and BoT additionally incorporate the geometric magnitude of disagreement. This provides a more direct surrogate for the pairwise component of gradient variance. To formalize this intuition, assume that the expected score-function gradient associated with an answer representation is locally Lipschitz in the semantic embedding space, i.e.,
\[
\|\mu_i^g-\mu_j^g\|^2 \le L_g^2 d(e_i,e_j),
\]
where $\mu_i^g=\mathbb{E}[\nabla_\theta \log \pi_\theta(y_i|x)]$ denotes the expected gradient direction induced by response $i$, $e_i$ is its normalized semantic embedding, and $d(e_i,e_j)=1-e_i^\top e_j$ is the cosine distance. Then the inter-response component of gradient variance satisfies
\[
V_{\mathrm{inter}}(x)
\;\lesssim\;
\frac{L_g^2}{2}
\sum_{i,j}\pi_i\pi_j d(e_i,e_j)
=
\frac{L_g^2}{2}\operatorname{CD}(x).
\]
This shows that, under a semantic-gradient smoothness assumption, CD upper-bounds the geometry-sensitive component of gradient variance up to a task-dependent constant. Unlike entropy, which replaces all pairwise disagreements by a worst-case constant, CD preserves the heterogeneous distances among responses and therefore yields a potentially tighter surrogate when disagreement is anisotropic.

Similarly, when responses form semantic clusters, BoT approximates the dispersion of cluster centroids around a barycentric consensus. Under the same smoothness assumption at the cluster level,
\[
V_{\mathrm{inter}}^{\mathrm{cluster}}(x)
\;\lesssim\;
L_g^2
\sum_k \Pi_k d(\mu_k,\mu^\star)
=
L_g^2 \operatorname{BoT}(x),
\]
indicating that BoT controls the mode-level component of gradient variance while being less sensitive to benign within-cluster paraphrastic variation.

Note that these bounds should be interpreted as structural justification rather than formal optimization guarantees: they rely on the assumption that semantic embedding distances are locally aligned with induced gradient directions. Our empirical analysis in Appendix~\ref{app:stats} validates this assumption statistically by showing stronger alignment between CD/BoT and sample-level gradient variance.

\section{Statistical Analysis of Uncertainty Measures}
\label{app:stats}

We conduct statistical experiments to examine whether different uncertainty measures align with the optimization-relevant quantity in GRPO, namely \emph{sample-level gradient variance}. Our objective is not to evaluate downstream accuracy, but to understand whether uncertainty reflects the variability of policy updates induced by rollout diversity.

\subsection{Data and Experimental Setup}

We randomly sample 1,000 queries from each of NarrativeQA and Qasper. For each query $q$, we generate a rollout group $\mathcal{Y}(q)=\{y_i\}_{i=1}^G$ from the policy $\pi_\theta(\cdot \mid q)$. Each rollout produces a score-function gradient
\[
g_i := \nabla_\theta \log \pi_\theta(y_i \mid q),
\]
and receives a scalar reward $r_i$. Following GRPO, normalized advantages $\hat{A}_i$ are computed within the group. To avoid the analysis being dominated by extreme cases, we remove the top 20 samples with the largest gradient variance and focus on the remaining dataset.

\subsection{Target Quantity: Sample-level Gradient Variance}

We define the \textbf{sample-level gradient variance} for a query $q$ as
\begin{equation}
\mathcal{V}(q)
:=
\frac{1}{G}
\sum_{i=1}^{G}
\left\|
\hat{A}_i g_i
-
\frac{1}{G}\sum_{j=1}^{G} \hat{A}_j g_j
\right\|_2^2.
\end{equation}

This quantity measures how much the update direction varies across sampled responses. When rollouts correspond to similar reasoning paths, gradients are aligned and $\mathcal{V}(q)$ is small. When rollouts reflect distinct reasoning strategies, gradients diverge and $\mathcal{V}(q)$ becomes large. Thus, $\mathcal{V}(q)$ captures both \emph{optimization instability} and \emph{learning signal diversity}.

\subsection{Uncertainty Measures}

For each query, we compute four uncertainty measures:
\begin{itemize}
    \item \textbf{Token Entropy:} uncertainty over token distributions;
    \item \textbf{Semantic Entropy (SE):} entropy over clustered semantic outputs;
    \item \textbf{Cosine Dispersion (CD):} pairwise geometric disagreement;
    \item \textbf{Barycentric Transport (BoT):} transport-based dispersion over semantic modes.
\end{itemize}
Each produces a scalar score $u(q)$.

\subsection{Association with Gradient Variance}

We first evaluate whether uncertainty measures are monotonically aligned with gradient variance using the \textbf{Spearman rank correlation coefficient}. For a metric $u(q)$, let $r_u(q)$ denote its rank. Spearman correlation is defined as
\begin{equation}
\rho(u,\mathcal{V})
=
1
-
\frac{6 \sum_q (r_u(q)-r_{\mathcal{V}}(q))^2}{N(N^2-1)}.
\end{equation}

This metric captures ranking consistency rather than linear dependence. A higher $\rho$ indicates that the uncertainty measure orders samples similarly to gradient variance.

All four uncertainty measures exhibit statistically significant correlation with $\mathcal{V}(q)$ (all $p \ll 0.05$), confirming that uncertainty is generally related to gradient variability. However, the magnitude of correlation differs substantially: geometry-aware measures (CD and BoT) achieve higher $\rho$ than entropy-based measures, indicating stronger alignment with optimization-relevant variation.

\subsection{Identification of High-Variance Samples}

We next evaluate whether uncertainty measures can identify the most optimization-relevant samples. Let $\mathcal{H}$ denote the top $10\%$ of samples ranked by $\mathcal{V}(q)$.

We evaluate two metrics:

\paragraph{AUC.}
We treat $u(q)$ as a scoring function for classifying $q \in \mathcal{H}$. The AUC is
\begin{equation}
\mathrm{AUC}
=
\mathbb{P}\big(u(q^+) > u(q^-)\big),
\end{equation}
where $q^+ \in \mathcal{H}$ and $q^- \notin \mathcal{H}$.

\paragraph{Top-$10\%$ Precision.}
\begin{equation}
\mathrm{Precision}@10\%
=
\frac{|\{q \in \mathcal{H} : q \in \text{Top-}10\%(u)\}|}{|\text{Top-}10\%(u)|}.
\end{equation}

These metrics are primarily \emph{descriptive}: they quantify how well each uncertainty measure ranks and retrieves high-variance samples. Unlike correlation, they are not inherently associated with a single null hypothesis.

Nevertheless, paired bootstrap analysis (described below) confirms that BoT achieves a statistically reliable improvement in AUC over CD, indicating that transport-based uncertainty better identifies the most unstable samples.

\subsection{Statistical Robustness via Paired Bootstrap}

To evaluate whether differences between methods are statistically stable, we perform \textbf{paired bootstrap resampling}. For each replicate, we resample queries with replacement and compute
\[
\Delta^{(b)} = \rho(\mathrm{BoT}, \mathcal{V}) - \rho(u, \mathcal{V}),
\]
for each baseline $u$.

All confidence intervals for $\Delta$ lie strictly above zero, indicating that BoT consistently outperforms entropy, SE, and CD in terms of correlation with gradient variance. This confirms that the observed advantage is not driven by a small subset of samples but persists under resampling.

\subsection{Held-out Gradient Prediction}

Finally, we evaluate whether uncertainty measures generalize to unseen data. For each metric $u(q)$, we fit a regression model
\[
\hat{\mathcal{V}}(q) = \beta_0 + \beta_1 u(q)
\]
and evaluate using cross-validation.

We report mean absolute error (MAE) and Spearman correlation on held-out folds. These metrics serve as \emph{descriptive predictive diagnostics}. Since we do not perform paired significance tests across folds, we do not interpret differences as statistically significant. Instead, we use them to assess whether the relationship between uncertainty and gradient variance generalizes beyond the observed samples.

\subsection{Discussion: Dataset-dependent Effects}

We observe that the alignment between uncertainty and gradient variance is weaker on Qasper than on NarrativeQA (Figure~\ref{fig:prelim}). One plausible explanation is that Qasper are more challenging, which exhibits higher semantic diversity with less reward differentiation: many rollouts are semantically distinct but receive similarly low rewards. In this regime, semantic dispersion does not translate as directly into gradient variability. In contrast, NarrativeQA often involves concise answer uncertainty, where different rollouts correspond to competing semantic hypotheses with clearer reward differences. This leads to a stronger coupling between semantic structure and gradient variance. As a result, while geometry-aware uncertainty remains beneficial on Qasper, its effect is attenuated compared to NarrativeQA. This observation is consistent with the smaller but still positive gains observed in the main optimization experiments.

\subsection{Final Takeaway}

Across both NarrativeQA and Qasper, geometry-aware uncertainty measures generally align more closely with gradient variance than entropy-based measures, although the strength of this relationship varies across datasets. This advantage is more pronounced on NarrativeQA, while on Qasper the relative improvements are smaller and less consistently significant across all metrics.

More broadly, the results support the theoretical claim that uncertainty in reasoning tasks is inherently \emph{anisotropic} and \emph{reward-coupled}: effective uncertainty measures should reflect the geometric diversity of reasoning paths and their interaction with learning dynamics, rather than relying solely on probability dispersion.

\section{Implementation Details}
\label{app:implementation}

\subsection{Benchmarks}
\label{app:benchmarks}

\noindent\textbf{NarrativeQA}~\citep{kovcisky2018narrativeqa} is an open-ended question answering benchmark constructed from books and movie scripts. Unlike extractive QA datasets, answers are not directly grounded in short spans, requiring models to synthesize information from long narratives. The open-ended nature of responses leads to multiple semantically valid answers, making it well-suited for studying response-level variability and uncertainty.

\noindent\textbf{Qasper}~\citep{dasigi2021dataset} is a question answering dataset over scientific articles, where questions often require synthesizing evidence from multiple sections of a paper. It features diverse answer types, including free-form generation and extractive spans, and introduces substantial ambiguity due to incomplete or partially supported evidence. This makes Qasper a challenging benchmark for evaluating reasoning under uncertainty.

\noindent\textbf{HotpotQA}~\citep{yang2018hotpotqa} is a multi-hop QA benchmark designed to evaluate reasoning across multiple documents. Questions require combining information from different supporting passages, with annotated evidence for explainability. Compared to open-ended QA, its answer space is more constrained, providing a complementary setting for studying uncertainty in structured reasoning tasks.

\noindent\textbf{MATH500}~\citep{hendrycks2021measuring} is a subset of the MATH dataset focused on mathematical problem solving. It contains competition-level problems spanning algebra, geometry, and number theory, where answers are deterministic and correctness depends on precise reasoning. This benchmark tests whether uncertainty signals remain informative under structured and rule-based reasoning.

\noindent\textbf{AIME24} is a collection of recent American Invitational Mathematics Examination problems, widely used for evaluating mathematical reasoning in LLMs. Problems typically admit a single numerical answer but require multi-step derivations, making them sensitive to small reasoning errors despite low observable response diversity. In this paper, we use the model trained on OlympiadBench to test AIME24.

\noindent\textbf{OlympiadBench}~\citep{huang2024olympicarena} consists of challenging competition-level problems designed to evaluate advanced reasoning ability. Compared to MATH500 and AIME24, it features higher complexity and often admits multiple valid solution strategies, introducing semantic diversity even under deterministic answer formats.

\subsection{Baselines}
\label{app:baselines}

\noindent\textbf{GRPO}~\citep{shao2024deepseekmath} serves as the base policy optimization framework without uncertainty modulation. It performs group-based updates by normalizing rewards across sampled responses, implicitly assuming all rollouts contribute equally regardless of their variability or informativeness.

\noindent\textbf{Semantic Entropy (SE). }Semantic entropy-based methods, including Seed-GRPO~\citep{chen2025seed} and EMPO~\citep{zhang2025right}, estimate uncertainty by clustering responses into semantic groups and measuring the entropy over cluster distributions. High entropy indicates diverse reasoning paths, which is used to modulate optimization updates. However, such methods primarily capture inter-response dispersion and may overlook structural differences within clusters.

\noindent\textbf{Kernel-based Entropy (KLE). }Kernel-based entropy~\citep{nikitin2024kernel} refines uncertainty estimation by incorporating similarity between responses via kernel functions. Instead of discrete clustering, it computes a continuous entropy measure over pairwise similarities, enabling smoother estimation of semantic diversity.

\noindent\textbf{Consistency-based Methods. }GRPO-CARE~\citep{chen2025grpo} measures uncertainty through agreement across sampled responses. The core intuition is that consistent outputs indicate high confidence, while disagreement reflects uncertainty. This approach avoids explicit entropy estimation but relies on pairwise agreement as a proxy signal.

\subsection{Training Details}
\label{appendix:Training}
\noindent\textbf{Split and Configuration. } We use designed split for training and testing. AIME does't have training set, we use the model trained on OlympiadBench to test AIME24. Unless otherwise specified, we train the models with 2000 steps, with a learning rate of 5e-5, and beta 0.002, with 16 rollouts and a temperature of 0.9. The alpha is set to 0.6 for default, with more ablations to see in Figure~\ref{fig:ablation}. By default, we train all models on Qwen 2.5-1.5b-instruct. And for math reasoning, we use the full solution instead of the final answer to calculate embeddings, and then uncertainty. Finally, we use PSC Bridge HPC and do the all the training with single H100-GPU. For more configure problems, please refer to our supplement materials.

\noindent\textbf{Semantic Clustering and Embedding Details.}
For CD and BoT, we use \texttt{all-MiniLM-L6-v2} as the embedding function $f_{\mathrm{emb}}$, implemented with \texttt{SentenceTransformer}. CD does not use entailment-based clustering. Given a group of $G$ sampled responses, we embed each response, $\ell_2$-normalize the embeddings, construct the cosine-distance matrix $D_{ij}=1-\cos(e_i,e_j)$, clip distances to $[0,1]$, and compute as discussed in methodology. BoT uses the same embedding model but first aggregates responses into semantic clusters using NLI. The entailment model is \texttt{roberta-large-mnli}. The default entailment threshold is $0.35$. Clustering is greedy and order-dependent: the first response initializes the first cluster; each subsequent response $h_i$ is compared against the current cluster representatives, where each representative is the first response assigned to that cluster. For cluster $k$, the NLI input is
$(\mathrm{premise}=\mathrm{rep}_k,\ \mathrm{hypothesis}=h_i)$. The response is assigned to the cluster with the largest entailment probability if that probability is at least $0.35$; otherwise a new cluster is created. If the NLI model exposes a label named \texttt{entailment}, we use that label index; otherwise we use the standard MNLI entailment index 2. After clustering, BoT computes a weighted centroid for each semantic cluster from
the response embeddings. Let $\mu_k$ denote the normalized centroid of cluster $k$ and $\pi_k$ its empirical mass. We compute as discussed in methodology with the final value clipped to $[0,1]$. The text passed to CD/BoT is task-dependent. For QA-style tasks
(\texttt{f1}, \texttt{f1\_bleu}, and \texttt{f1\_em}), we use the extracted free-form answer text. For math/exact-match tasks, we use the full solution (\texttt{use\_full\_completion\_for\_clustering=true}). In
our MATH constrained CD/BoT runs, this flag is enabled, so the full rollout text is embedded and clustered.

\noindent\textbf{Reward definitions and RD normalization.}
Reward dispersion is computed per prompt group from the realized rollout rewards. For a configured reward range $[r_{\min},r_{\max}]$, the maximum possible value is
\[
\mathrm{RD}_{\max}(G)
=
\frac{2}{G}
\left\lfloor \frac{G}{2} \right\rfloor
\left\lceil \frac{G}{2} \right\rceil
(r_{\max}-r_{\min}).
\]
We then normalize as
\[
\mathrm{RD}_{\mathrm{norm}}(q)
=
\mathrm{clip}
\left(
\frac{\mathrm{RD}_{\mathrm{raw}}(q)}
{\mathrm{RD}_{\max}(G)},
0,1
\right).
\]
In the constrained CD/BoT experiments, we set $r_{\min}=0$ and $r_{\max}=2$ in the configuration. Thus, for the common group size $G=16$.

\begin{table*}[t]
\centering
\begin{tabular}{lll}
\toprule
Benchmark & Metric / extraction & Reward range for RD \\
\midrule
Qasper &
$2.0 \times$ token-F1 over references &
$[0,2]$ \\
NarrativeQA &
$2.0 \times$ token-F1 over references &
$[0,2]$ \\
HotpotQA &
$2.0 \times$ QA F1/EM-style score &
$[0,2]$ \\
MATH &
$2.0$ if extracted answer matches, else $0$ &
$[0,2]$ \\
AIME24 &
$2.0$ if extracted answer matches, else $0$ &
$[0,2]$ \\
OlympiadBench &
$2.0$ if extracted answer matches, else $0$ &
$[0,2]$ \\
\bottomrule
\end{tabular}
\caption{Per-task reward definitions used for reward-dispersion normalization.}
\label{tab:reward-ranges}
\end{table*}

For QA benchmarks, the reward is continuous because it is proportional to token F1 over the reference answers. For math-style benchmarks, the correctness reward is binary after answer extraction and normalization. In all constrained CD/BoT runs, RD normalization uses the configured correctness range $[0,2]$.

\section{Discussions} 
\label{app:discussion}

\subsection{Trade-off between Uncertainty and Reward Dispersion}
\label{app:rd_discussion}

A key design challenge is balancing the suppression of spurious variance with the preservation of informative learning signals. Pure uncertainty minimization tends to collapse response diversity, suppressing high-variance samples that often carry strong gradient signals, as discussed in the calibration gap. In contrast, relying solely on reward dispersion encourages large reward variability, which may arise from ambiguous or adversarial prompts and can destabilize optimization.

To address this tension, we introduce a complementary modulation scheme. Geometry-aware uncertainty is defined as$w_{\mathrm{UQ}}(q) = 1 - \alpha_G \mathrm{UQ}(q)^2$,
while reward dispersion is scaled as
$w_{\mathrm{RD}}(q) = 1 + \alpha_G \mathrm{RD}(q)$,
where $\mathrm{RD}(q)$ is the normalized reward dispersion: $\mathrm{RD}(q) = \mathrm{clip}\!\left(
\frac{\sum_i |r_i - \bar{r}|}{\mathrm{RD}_{\max}(G)}, 0, 1
\right)$. This formulation induces a calibrated regime where moderate variability is amplified, while extreme dispersion is suppressed. Intuitively, geometry-aware uncertainty acts as a reliability constraint that penalizes semantically inconsistent or noisy rollout groups, whereas reward dispersion serves as an informativeness signal that promotes updates when reward differentiation is meaningful.

A potential concern is that reward dispersion may amplify groups with noisy or mis-specified rewards. This is mitigated by the interaction between the two components. While RD alone would favor high-variance groups, the uncertainty term down-weights cases where semantic structure is inconsistent, preventing amplification of spurious variability. In addition, RD is bounded within $[0,1]$, yielding $w_{\mathrm{RD}}(q) \in [1, 1+\alpha_G]$, which ensures that gradient scaling remains controlled. Furthermore, RD captures relative reward differences within a group, so uniformly noisy or biased rewards tend to cancel out after normalization.

The normalization and scaling are designed to ensure stable behavior across different rollout configurations. Normalizing by $\mathrm{RD}_{\max}(G)$ makes RD invariant to reward scale, while the group-dependent factor $\alpha_G = \alpha / \log G$ compensates for the natural increase of dispersion with group size. From an optimization perspective, the combined modulation
\[
\tilde{A}_i = A_i \cdot w_{\mathrm{UQ}}(q) \cdot w_{\mathrm{RD}}(q)
\]
ensures that the gradient estimator remains well-conditioned, as both amplification and suppression are bounded.

Overall, this design establishes a structured trade-off: updates are strengthened when responses exhibit aligned semantic structure with informative reward contrast, and attenuated when variance is dominated by noise. Rather than uniformly suppressing or amplifying gradient variance, GCPO selectively calibrates it according to both semantic geometry and reward informativeness, leading to more stable and effective policy optimization. Empirically, we observe that when semantic structure and reward variation are well aligned, this interaction yields consistent gains, while in settings where such alignment weakens (e.g., OlympiadBench), the effect naturally diminishes.

\subsection{Relation to KL-constrained Optimization}
Recent work such as GVPO~\cite{zhang2025gvpo} derives policy update weights from the analytical solution of KL-constrained reward maximization, providing a principled objective-level formulation with strong theoretical guarantees. Specifically, GVPO directly incorporates the closed-form relationship between rewards and optimal policies under a KL constraint into its gradient weights, ensuring alignment with the KL-constrained optimum and offering a variance–covariance interpretation of the update. 

In contrast, our approach addresses a complementary and orthogonal problem: how to \emph{selectively modulate gradient contributions across samples} given a fixed optimization objective. While KL regularization (as used in GRPO and methods such as GVPO) controls the \emph{magnitude and stability of policy updates at the distribution level}, it does not differentiate between informative and uninformative rollout groups within a batch. In this paper, as analyzed in Section~\ref{sec:preliminary}, we focus on the core challenge in GRPO-style training of the heterogeneity of gradient variance across queries, where some samples provide rich learning signals while others introduce spurious noise. 

GCPO operates precisely at this level by introducing uncertainty-aware weights that calibrate gradient contributions based on both semantic structure and reward informativeness. In this sense, KL-based methods and GCPO address different axes of the optimization problem: KL constraints ensure global policy stability, whereas our geometry-aware uncertainty measures provide \emph{fine-grained, sample-level control} over gradient variance. These mechanisms are therefore complementary rather than competing.

We do not include KL-constrained optimization methods as direct baselines, as they fundamentally alter the optimization objective and gradient formulation. In contrast, our goal is to study uncertainty as a \emph{plug-in signal} within the GRPO framework, isolating its role in shaping optimization dynamics. Our contribution is therefore not to propose the most optimal RL objective, but to provide a focused analysis of how uncertainty interacts with gradient variance and learning signals. Through this lens, we identify key limitations of existing uncertainty measures and derive practical design principles for aligning uncertainty estimation with effective optimization.

\subsection{Bias induced by uncertainty weighting.}
GCPO modifies the standard GRPO estimator through query-level reweighting, and therefore does not preserve the unbiasedness of the original gradient estimator. Recall that GRPO estimates
\[
\hat{g}(x) = \frac{1}{G} \sum_{i=1}^{G} \hat{A}_i \nabla_\theta \log \pi_\theta(y_i \mid x).
\]
GCPO introduces a query-level weight $w(q) = \omega_{\mathrm{geo}}(q)\,\omega_{\mathrm{RD}}(q)$, yielding
\[
\tilde{g}(x) = w(q)\,\hat{g}(x), \quad 
\mathbb{E}[\tilde{g}(x)] = \mathbb{E}[w(q)\,\hat{g}(x)] \neq \mathbb{E}[\hat{g}(x)].
\]
Thus, GCPO optimizes a reweighted objective where each query contributes proportionally to its estimated reliability and informativeness.

This reweighting induces a structured bias but also directly affects gradient variance:
\[
\mathrm{Var}(\tilde{g}(x)) = \mathrm{Var}(w(q)\,\hat{g}(x))
\approx \mathbb{E}[w(q)^2 \mathrm{Var}(\hat{g}(x)\mid q)] 
+ \mathrm{Var}(w(q))\,\|\mathbb{E}[\hat{g}(x)\mid q]\|^2.
\]
Since $\omega_{\mathrm{geo}}(q)$ suppresses semantically inconsistent (high-variance) groups and $\omega_{\mathrm{RD}}(q)$ amplifies reward-informative ones, GCPO effectively reduces variance from noisy directions while preserving informative gradient components.

This behavior is supported empirically. As shown in Section~\ref{sec:main_experiment}, uncertainty-aware methods consistently reduce training variance compared to raw GRPO, and GCPO further improves stability while maintaining performance gains. Geometry-aware measures also exhibit stronger alignment with sample-level gradient variance, leading to lower standard deviations across runs. 

Therefore, GCPO should be interpreted as a calibrated estimator that trades unbiasedness for improved variance control and learning signal quality. Rather than strictly optimizing the original GRPO objective, it reshapes the effective training distribution toward informative rollout groups. A formal analysis of convergence under such adaptive weighting remains an important direction for future work.

\subsection{On the Semantic--Gradient Alignment Assumption}
\label{app:assumption}

Our theoretical formulation relies on a structural assumption that semantic distances in the embedding space are locally aligned with differences in induced gradient directions. Formally, this is captured by the Lipschitz-type condition used in Appendix~\ref{app:link}, which relates embedding distance to gradient discrepancy.

We emphasize that this assumption is not intended as a strict guarantee, but as a modeling hypothesis that enables a tractable connection between response geometry and optimization dynamics. Importantly, our framework does not depend on this assumption holding universally; rather, it requires that semantic structure provides a better proxy for gradient variability than purely distributional measures such as entropy.

Empirically, our results support this premise. As shown in Figure~\ref{fig:prelim} and Appendix~\ref{app:stats}, geometry-aware measures (CD, BoT) consistently exhibit stronger alignment with sample-level gradient variance across QA benchmarks. At the same time, we observe that when this alignment weakens, as in OlympiadBench (Table~\ref{tab:gradient}), the performance gains of GCPO become less pronounced. This behavior is consistent with the theoretical expectation and highlights that the effectiveness of geometry-aware uncertainty depends on task-specific alignment between semantic variation and optimization signals.

Overall, this assumption should be viewed as a domain-dependent inductive bias rather than a universal property. Our results suggest that when semantic embeddings capture optimization-relevant variation, geometry-aware uncertainty provides a more faithful surrogate for gradient variance than entropy-based measures.

\subsection{Computational Overhead Analysis}
\label{app:complexity}

We analyze the computational overhead introduced by GCPO relative to standard GRPO. The dominant cost in GRPO-style training arises from rollout generation, which requires multiple forward passes of the language model. In comparison, the additional operations introduced by GCPO are lightweight and scale only with the rollout group size.

Specifically, given a group of size $G$, computing geometry-aware measures involves: (i) embedding extraction $O(Gd)$, (ii) pairwise cosine distance computation $O(G^2)$ for CD, and (iii) cluster-level aggregation for BoT. Reward dispersion requires only $O(G)$ operations. In practice, $G$ is small (e.g., $G \leq 16$), making these computations negligible relative to model inference.

Empirically, we observe no noticeable increase in wall-clock training time compared to standard GRPO under identical rollout configurations. This suggests that GCPO provides improved optimization signals with minimal additional computational overhead.

\subsection{On the OT Interpretation of BoT.}
We hope to emphasize that Barycentric Transport (BoT) is not intended as an exact optimal transport (OT) computation. Rather, it is a lightweight, geometry-aware surrogate inspired by the OT perspective of aligning distributions toward a consensus representation. Specifically, we approximate the transport structure via cosine-based deviations from a barycentric direction, which avoids the computational overhead of solving OT problems while preserving sensitivity to mode-level semantic dispersion.

Our goal is not to approximate OT distances faithfully, but to construct a tractable measure that better aligns with optimization-relevant gradient variance. Under the same semantic-gradient smoothness assumption used throughout the paper, BoT provides a structured proxy for the inter-mode component of gradient variance. This is supported by empirical results in Appendix~D, where BoT consistently exhibits stronger alignment with sample-level gradient variance than entropy-based measures.

We acknowledge that BoT depends on clustering and embedding quality, which may introduce sensitivity in certain settings. However, this dependence is inherent to all response-level uncertainty methods operating in semantic space. Importantly, our empirical analysis demonstrates that despite this approximation, BoT remains robust and effective across diverse benchmarks, indicating that precise OT computation is not necessary for capturing optimization-relevant structure.

\begin{table*}[t]

\renewcommand{\arraystretch}{1.15}

\centering

\resizebox{\textwidth}{!}{

\begin{tabular}{c ccc ccc}

\toprule

\multirow{2}{*}{\textbf{Method}} &

\multicolumn{3}{c}{\textbf{NarrativeQA}} &

\multicolumn{3}{c}{\textbf{Qasper}} \\

\cmidrule(lr){2-4}\cmidrule(lr){5-7}

& \textbf{Acc} & \textbf{BLEU} & \textbf{F1}

& \textbf{Acc} & \textbf{BLEU} & \textbf{F1} \\

\midrule\midrule

Q-Hawkeye (2026)

& 71.13{\tiny$\pm$0.40} & 52.44{\tiny$\pm$0.40} & 71.63{\tiny$\pm$0.26}

& 12.98{\tiny$\pm$1.08} & 10.94{\tiny$\pm$0.44} & 23.50{\tiny$\pm$1.01} \\

EG-SPO (2026)

& 71.51{\tiny$\pm$1.12} & 52.16{\tiny$\pm$0.88} & 71.84{\tiny$\pm$0.78}

& \textbf{14.29}{\tiny$\pm$1.22} & \textbf{12.40}{\tiny$\pm$1.03} & 23.91{\tiny$\pm$0.83} \\

R2VPO (2026)

& 72.32{\tiny$\pm$0.24} & 52.12{\tiny$\pm$1.18} & 72.52{\tiny$\pm$0.23}

& 13.02{\tiny$\pm$0.28} & 9.67{\tiny$\pm$0.72} & 23.45{\tiny$\pm$0.46} \\

\midrule

GRPO + CD

& \underline{72.59}{\tiny$\pm$0.48} & \underline{53.72}{\tiny$\pm$0.89} & \underline{73.25}{\tiny$\pm$0.27}

& \underline{13.67}{\tiny$\pm$0.33} & \underline{11.28}{\tiny$\pm$0.83} & \underline{24.25}{\tiny$\pm$0.17} \\

GRPO + BoT

& \textbf{73.77}{\tiny$\pm$0.16} & \textbf{54.85}{\tiny$\pm$0.74} & \textbf{73.82}{\tiny$\pm$0.07}

& 13.91{\tiny$\pm$0.63} & 10.94{\tiny$\pm$1.49} & \textbf{24.37}{\tiny$\pm$0.64} \\

\bottomrule

\end{tabular}

}

\caption{Comparison with the most recent variance/uncertainty-aware baselines on NarrativeQA and Qasper. Best results are in \textbf{bold}, second best are \underline{underlined}.}

\label{tab:new_baselines}

\vspace{-10pt}

\end{table*}

\section{Additional Experiments}
\label{app:additional}

\noindent\textbf{Extended Baseline Comparison.}
Our primary comparisons focus on methods that are directly aligned with the scope of this work, namely response-level uncertainty modulation within the GRPO framework. These baselines are most appropriate for evaluating the specific question we study: whether uncertainty signals can effectively regulate gradient variance during group-based policy optimization. 

At the same time, we acknowledge the broader landscape of uncertainty- and variance-aware methods in policy optimization. To address this, we additionally include three most recent baselines:Q-Hawkeye \cite{xie2026q}, EG-SPO \cite{hu2026entropy}, and R\textsuperscript{2}VPO \cite{luo2026ratio}, (Table~\ref{tab:new_baselines}) which, while not strictly aligned with our formulation, provide complementary perspectives on uncertainty and stability. These methods operate at different levels of the optimization pipeline, including sample-level reweighting, token-level routing, and policy-level variance control, respectively. Evaluating them in our setting offers additional insight into how different forms of uncertainty and variance interact with GRPO-style training.

\noindent\textbf{Adapted Uncertainty and Variance-Control Baselines. }
To ensure fair comparison, all methods are adapted into a unified GRPO framework. Since these approaches are originally proposed under different objectives or training pipelines (e.g., PPO-based formulations or domain-specific settings), we reinterpret each method as a plug-in gradient modulation scheme applied on top of the GRPO update. This design isolates the effect of uncertainty or variance signals without altering the underlying optimization objective.

\noindent\textbf{Q-Hawkeye-style Uncertainty Reweighting. }
Q-Hawkeye introduces uncertainty-aware dynamic optimization by estimating uncertainty from variability in predicted scores and using it to reweight policy updates. We include this baseline as it is conceptually closest to our setting, representing sample-level uncertainty modulation within a GRPO-like framework. To adapt it to LLM QA, we define uncertainty as the variance of rollout rewards within each group, i.e., $u(q)=\mathrm{Var}(r_1,\dots,r_G)$, and apply a linear reweighting $w(q)=1-\alpha \tilde u(q)$, where $\tilde u(q)$ is a normalized variance. This adaptation captures the core idea of uncertainty-based reweighting while removing domain-specific components from the original multimodal setting.

\noindent\textbf{EG-SPO-style Entropy Gating. }
EG-SPO proposes entropy-guided routing that modulates gradient magnitude based on token-level predictive uncertainty while preserving update direction. We include it as a closely related approach that also interprets uncertainty as a control signal for gradient scaling. Since EG-SPO is originally built on PPO with token-level updates and a multi-stage pipeline, we adopt a simplified variant by aggregating token-level entropy into a response-level measure and using it to gate the group-normalized advantage in GRPO. This allows for a controlled comparison between entropy-based modulation and our geometry-aware uncertainty signals.

\noindent\textbf{R\textsuperscript{2}VPO-style Variance Damping. }
R\textsuperscript{2}VPO improves stability by controlling policy-ratio variance, replacing hard clipping with a smoother variance-based constraint. Although not an uncertainty-based method, it is highly relevant as it targets variance from a complementary source. We include a lightweight approximation that penalizes high policy-ratio variance at the response level, using $w_i = 1/(1+\lambda \mathrm{Var}_t(\rho_{i,t}))$. This preserves the core intuition of variance control without modifying the GRPO objective, enabling direct comparison with GCPO’s approach to regulating gradient variance induced by rollout diversity and reward differentiation.

\noindent
Overall, these additional baselines span multiple levels of the optimization pipeline. While they are not designed for response-level uncertainty modeling, incorporating them under a unified framework provides a broader perspective and further highlights the distinct role of geometry-aware and reward-calibrated uncertainty in GCPO.




\end{document}